\documentclass[letterpaper, 10 pt, conference]{ieeeconf}  

\IEEEoverridecommandlockouts                              

\usepackage{amssymb}
\usepackage{graphicx}
\usepackage{subfigure}
\usepackage{mathtools}
\usepackage{amsmath}
\usepackage{multirow}
\usepackage{booktabs}
\usepackage{makecell}

\usepackage{bbm}
\usepackage{dsfont}
\usepackage{times}

\usepackage{multicol}

\usepackage{xcolor}
\usepackage{amssymb}
\usepackage{graphicx}
\usepackage{mathtools}
\usepackage{amsmath}
\usepackage{gensymb}
\usepackage{float}
\usepackage{multirow}
\usepackage{makecell}
\usepackage{hyperref}
\hypersetup{
    colorlinks=true,
    linkcolor=blue,
    filecolor=magenta,      
    urlcolor=blue,
    pdftitle={Overleaf Example},
    pdfpagemode=FullScreen,
    }

\DeclarePairedDelimiter\floor{\lfloor}{\rfloor}
\usepackage{tabularx}
    \newcolumntype{L}{>{\raggedright\arraybackslash}X}
\usepackage[utf8]{inputenc}
\usepackage[english]{babel}

\usepackage{amsthm}

\usepackage{subcaption}

\newtheorem{lemma}{Lemma}[section]

\title{\LARGE \bf  TOPGN: Real-time Transparent Obstacle Detection using Lidar Point Cloud Intensity for Autonomous Robot Navigation 
}

\author{Kasun Weerakoon, Adarsh Jagan Sathyamoorthy, Mohamed Elnoor, Anuj Zore, and Dinesh Manocha.\\
\small{Supplemental version including Tech Report, and Video at \url{http://gamma.umd.edu/topgn/}}}

\newcommand{\blue}[1]{{\color{black}#1}}
\newcommand{\rev}[1]{\textcolor{black}{#1}}

\begin{document}

\maketitle
\thispagestyle{empty}
\pagestyle{empty}

\begin{abstract}
We present TOPGN, a novel method for real-time transparent obstacle detection for robot navigation in unknown environments. We use a multi-layer 2D grid map representation obtained by summing the intensities of lidar point clouds that lie in multiple non-overlapping height intervals. We isolate a neighborhood of points reflected from transparent obstacles by comparing the intensities in the different 2D grid map layers. Using the neighborhood, we linearly extrapolate the transparent obstacle by computing a tangential line segment and use it to perform safe, real-time collision avoidance. Finally, we also demonstrate our transparent object isolation's applicability to mapping an environment. We demonstrate that our approach detects transparent objects made of various materials (glass, acrylic, PVC), arbitrary shapes, colors, and textures in a variety of real-world indoor and outdoor scenarios with varying lighting conditions. We compare our method with other glass/transparent object detection methods that use RGB images, 2D laser scans, etc. in these benchmark scenarios. We demonstrate superior detection accuracy in terms of F-score improvement at least by 12.74\% and 38.46\% decrease in mean absolute error (MAE), improved navigation success rates (at least two times better than the second-best), and a real-time inference rate ($\sim 50$ Hz on a mobile CPU). We will release our code and challenging benchmarks for future evaluations upon publication.

\end{abstract}

\section{Introduction} \label{sec:intro}

\begin{figure}[t]
    \centering
    \includegraphics[width=0.9\columnwidth,height=5.8cm]{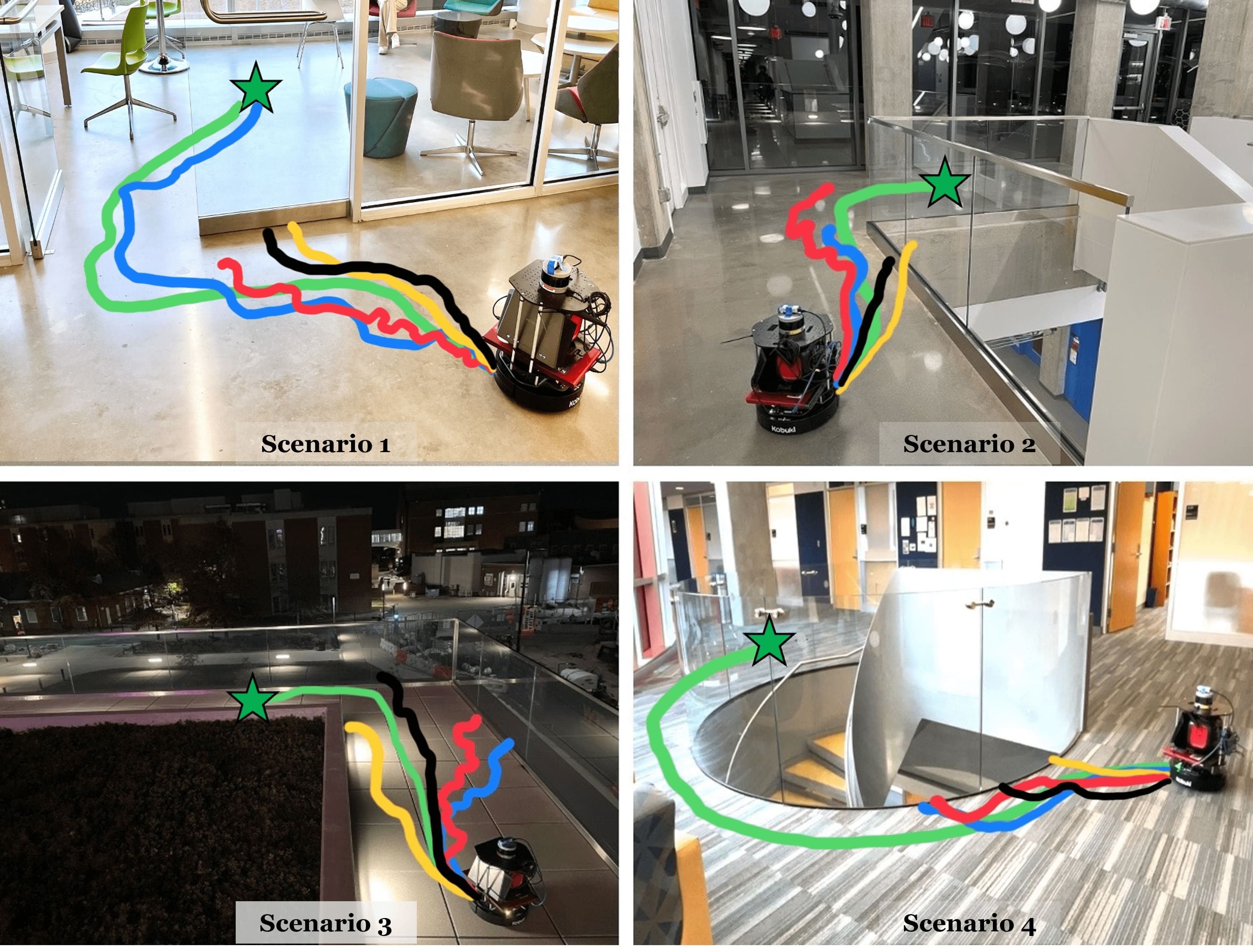}
    \caption{\small{Our method can robustly detect transparent obstacles in scenes with varying illumination in real-time ($\sim 50$ Hz). The figure shows the robot's trajectories in one trial when a planner \cite{fox1997dwa} used our method (in green), GDNet \cite{gdnet} (in blue), Translab \cite{translab} (in red), Glass-SLAM \cite{glass-slam} (in yellow), and 2D laser scans (in black) to detect transparent obstacles in unknown environments. RGB segmentation methods \cite{gdnet,translab} are affected by strong lighting changes, and motion blur in these environments, causing collisions with glass, or freezing during navigation. SLAM methods such as \cite{glass-slam} require $\sim 3$ seconds to update the locally sensed obstacles on to a map, leading to collisions. Our method's accurate detection of transparent and opaque obstacles facilitates safe, collision-free navigation in unmapped, unknown environments.}}
    \label{fig:cover_image}
    \vspace{-10pt}
\end{figure}

Modern buildings are filled with glass walls, full-height windows, or other transparent obstacles that are challenging for a robot to detect and autonomously navigate around. This is primarily because a majority ($> 90\%$) of the light energy incident on transparent objects gets transmitted through them \cite{visagge} and only a small amount gets reflected for robot-mounted sensors (e.g. RGB, depth cameras, 2D/3D lidars) to detect. The inability of a robot to accurately perceive and avoid transparent obstacles in real-time could lead to serious collisions, which can result in damage to the robot and its environment \cite{visagge,lasitha2022cartographer_glass,gdnet,translab}.  

To address this challenge, there have been several methods for glass/transparent object detection using various sensor modalities such as RGB \cite{gdnet,translab}, depth \cite{depth-aware-glass-surface-detection,rgb-d-depth-completion}, and thermal images \cite{rgb-thermal}, ultrasound \cite{ultrasonic-and-rgbd}, 2D laser scans \cite{glass-slam}, 3D point clouds \cite{mim-arxiv}, etc. Methods utilizing RGB images typically semantically segment transparent objects \cite{gdnet}. Some works have combined RGB images with the depth of objects in the scene, or with the thermal signature \cite{rgb-thermal} of transparent objects to improve detection accuracy. Since RGB cameras and the images they capture are impacted by the environmental lighting conditions, the detection accuracy of RGB-based methods can deteriorate significantly in low-light or extremely bright conditions (see Fig. \ref{fig:cover_image}). This is further exacerbated if the environment contains multiple light sources and highly reflective surfaces.

Conversely, methods that use 2D scans or 3D point clouds from lidars are resilient to external lighting changes, and excessive reflections from various surfaces since lidars have an in-built light source. However, most of the light from lidars passes through transparent obstacles and $< 10\%$ of the intensity is reflected to the sensor at certain angles \cite{visagge}, making detection of the shape and size of the obstacle challenging. In addition, lidar-based methods \cite{visagge,glass-slam,lasitha2022cartographer_glass,foster2023reflectance} have been mainly limited to mapping tasks that typically require a robot to perform several loops in the environment to detect and map transparent obstacles. This makes them impractical for real-time navigation and collision avoidance.

\textbf{Main Contributions:} We present TOPGN (Transparent Obstacle Perception for Guidance and Navigation), a novel approach to accurately detect transparent obstacles in the environment, and extrapolate (predict regions where transparent obstacles could be present) them in 2D grid maps to avoid collisions preemptively. The novel components of our work include:

\begin{itemize}
    \item A novel method to accurately isolate \textit{transparent obstacle neighborhoods} (TONs) that arise in a multi-layer 2D grid map representation 
    \cite{mip-maps, mim-arxiv, deng2018multi, 723367, hierarchical-occlusion-maps} that are computed from 3D lidar point clouds. Our approach is based on identifying a Gaussian pattern exhibited by point cloud intensities while they strike transparent obstacles. Our approach has a low computational overhead and can execute real-time on mobile CPUs at $\sim 50$ Hz, and can also be used to create a 2D map of an environment containing transparent obstacles. Further, we observe superior transparent obstacle detection accuracy compared to existing RGB-based and lidar-based methods (at least 12.74\% improvement in F-score compared to the second-best performing method).

    \item A novel method to linearly extrapolate transparent obstacles from an instance of a TON such that the regions that could potentially contain transparent obstacles are encompassed. The robot trajectories that avoid extrapolation are guaranteed to avoid transparent obstacles in completely unknown environments. Our method can handle curved transparent obstacles, and we demonstrate that our method leads to superior navigation success rates (at least 2 times better than the second-best method) in reaching the robot's goal.   

    \item We implement our method on a real Turtlebot equipped with a Velodyne VLP16 lidar. We perform extensive real-world evaluation in challenging indoor and outdoor scenarios with severe lighting changes, reflective surfaces, and transparent obstacles with curved shapes. We create and also release our test benchmarks with these scenarios that contain RGB images, and 2D grid maps with transparent obstacle annotations. We demonstrate that our approach is robust in all these scenarios while existing methods deteriorate in terms of detection accuracy and lead to collisions.  
\end{itemize}
\section{Related Work}
\label{sec:rel-work}

In this section, we provide an overview of methods that use point cloud intensities, or some kind of images (RGB, depth, thermal) to perceive transparent obstacles.  

\subsection{Transparent Obstacle Detection using Lidars and Point Clouds}
There have been several approaches \cite{glass-slam,visagge,fusing-sonar,lasitha2022cartographer_glass} that use the intensity of lidar point clouds to detect glass and other transparent obstacles for SLAM (Simultaneous Localization And Mapping) applications. One of the earliest works in this domain is by Foster et al. \cite{visagge}, who proposed to track the angles of incidence on glass from which a lidar's laser rays are reflected with maximum intensity. They identified that $\sim 0^{\circ}$ incidence leads to high-intensity returns back to the lidar. \blue{In a recent extension, Foster et al. \cite{foster2023reflectance} proposed a more general approach that constructs a Reflectance Field Map (RFM) based on the reflectance of points from different perspectives. Next, they identify a distinct H-pattern that transparent obstacles cause in the RFM. \cite{foster2023reflectance} can robustly map glass in real-time even when the lidar is disturbed by bumps and suspension loading, and in the presence of dynamic pedestrians.}

Subsequent works such as \cite{glass-slam} use this theory to recognize the reflected light intensity profile on the glass to detect it, and construct a map using the particle filter. However, it requires several walk-throughs in the environment to map glass, and the resulting map could still miss some portions of glass. Weerakoon et al. \cite{lasitha2022cartographer_glass} improved this accuracy of map building using Graph SLAM. Tibebu et al. \cite{tibebu2021lidar_glass_detection} identified the changes in the distance and intensity measurements between neighboring point clouds to estimate the glass profile. Other works such as Wei et al. \cite{fusing-sonar} have augmented the distance (from the obstacle) information obtained from a lidar with the distance from ultrasound sensors to map environments with glass. However, such methods are limited by the short range of ultrasound sensors making them unsuitable for real-time collision avoidance.


Additionally, such SLAM methods cannot be directly used for real-time navigation in unknown environments since they require several seconds to construct the obstacles in the map. \blue{Our approach is based on the specular reflection at $0^{\circ}$ incidence but extrapolates transparent obstacles in real-time for navigating unknown, unmapped environments with transparent obstacles. Further, we demonstrate our approach's generality by also demonstrating its applicability in real-time mapping.}


\subsection{Transparent Obstacle Detection using Images}
There has been extensive work on glass and transparent object segmentation in RGB, depth, and thermal images. Huang et al. \cite{ultrasonic-and-rgbd} developed a wearable setup with a depth camera and ultrasound sensors to improve the depth measurement accuracy for glass detection to guide the visually impaired in real time. However, due to the low range and fields-of-view of these sensors, they cannot be reliably used for robot navigation. 

More recently, GDNet \cite{gdnet,large-field-contextual} released a large-scale glass detection dataset and proposed a semantic segmentation method for detecting large-sized glass from RGB images using the contextual features from a large receptive field. Similarly, TransLab \cite{translab} proposed using boundary cues as a means to improve large and small transparent objects. Lin et al. \cite{rich-context-aggregation} overcame the inaccuracies in GDNet and TransLab models (e.g. confusing open spaces as glass) by adding a module to refine glass detection by identifying reflections. This was later extended by integrating the missing depth data from glass in a depth image to detect the presence of glass surfaces \cite{depth-aware-glass-surface-detection}. Other methods \cite{rgb-thermal} have fused RGB with thermal images by using the fact that thermal energy is blocked by transparent objects while visible light passes through. 

The challenge posed by transparent objects has also been studied in regards to stereo matching \cite{corner-case-stereo}, object reconstruction and manipulator grasping \cite{6dof-pose-estimation,keypose,clear-grasp,multimodal-transfer-learning,rgb-d-depth-completion,dexnerf,monograspnet}. All these methods use depth from a time-of-flight depth camera \cite{6dof-pose-estimation}, a stereo camera \cite{keypose}, or fuse the slight discoloration observed in RGB images with depth information \cite{multimodal-transfer-learning} to detect graspable small 3D objects, assess their position and orientation, and plan a way to grasp them. Using neural radiance fields \cite{dexnerf} and additional lights to obtain more reflections to improve detections have also been proposed for this task. 

However, methods that use RGB, RGB-D, and other types of cameras suffer from several key downsides that make them inappropriate for real-time navigation. RGB/RGB-D cameras suffer from low range and fields-of-view compared to lidars. Further, the quality of images deteriorates as the environmental illumination sharply increases or decreases. Most of the segmentation methods for glass detection get confused by reflective surfaces and tend to classify them as glass. This could severely restrict a robot's notion of navigable free space causing undesirable behaviors such as halting/freezing \cite{frozone}.


\subsection{Multi-Layer Representations}
For robot navigation, 2D grid/cost maps \cite{qi2020building, tripathy2021care, de2017skimap} 
have been used as a standard data structure to represent the distribution of obstacles, and navigation costs in an environment. The robot's planner uses the costs in these maps to compute a least-cost, collision-free path or velocity to navigate to its goal. Multi-layer Image Representations have been widely used for image processing tasks such as instance retrieval \cite{deng2018multi}, image compression \cite{723367}, and interpretation \cite{ivasic2014multi}. Other
Multi-layer, hierarchical representations such as MIP-maps \cite{mip-maps}, hierarchical occlusion maps \cite{hierarchical-occlusion-maps}, quad trees \cite{quad-trees,linear-quad-trees}, multi-layer intensity maps \cite{mim-arxiv} have existed that use several layers of grid arrays to represent various applications in graphics rendering, and obstacle detection. Our approach reduces the dimensions of point cloud intensities to 2D using \cite{mim-arxiv}, and uses it for transparent obstacle detection and navigation. 
\section{Background} \label{sec:background}
In this section, we first define the symbols and notations used in our work, then explain the underlying intensity maps and the preliminary concepts used to detect transparent obstacles.
\subsection{Definitions and Assumptions}
We make the following assumptions in our formulation for transparent obstacle detection. We assume that a robot modeled as a cylinder of radius $r_{rob}$, and height $h_{rob}$ is equipped with a lidar mounted at height $h_{lid}$ (and $h_{lid} \le h_{rob}$) that shoots out light rays and generates 3D point clouds (PC) with their associated intensities. For simplicity, we assume that the robot's and the lidar's centers coincide. Each point in the point cloud is represented as $\mathbf{p} = \{x, y, z, int\}$, where $x, y, z$ denote the point's location relative to the lidar, and $int \in [0, i_{max}] $ denotes its intensity, and $i_{max}$ denotes the maximum possible intensity. 
Our coordinate frame convention is defined with the positive x, y, and z axes pointed forward, leftward, and upward respectively attached to the ground ($z = 0$) beneath the robot's center of mass. Throughout the text, symbols $j, k$ are used to denote indices, and $t$ denotes a time instant. 

\begin{figure*}[t]
    \centering
    \includegraphics[width=0.8\linewidth, height=8cm]{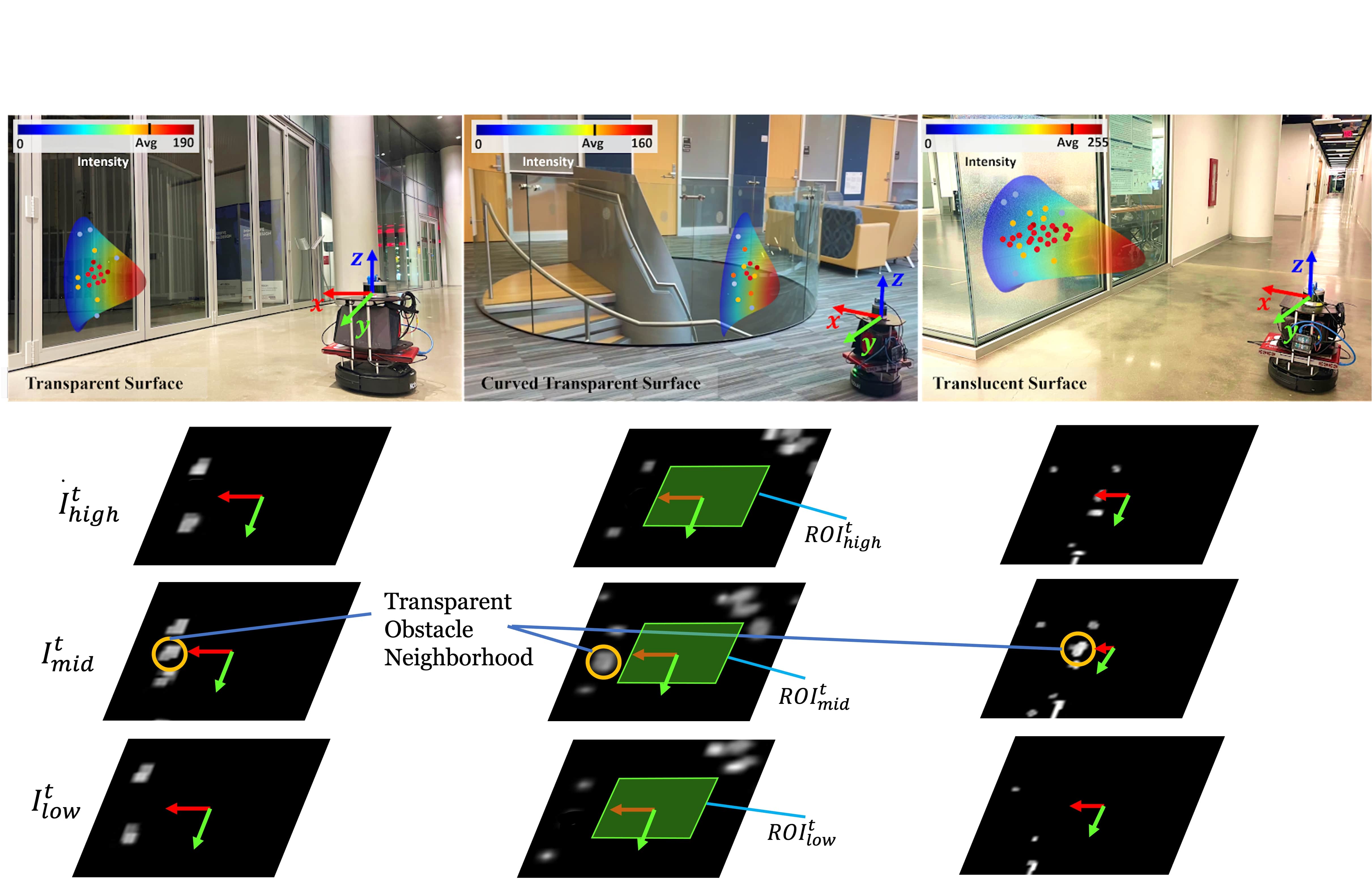}
    \caption{\small{The Gaussian distribution of point cloud intensities observed when the robot faces obstacles with different levels of transparency and shapes. The average intensities along with the distribution depend on the transparency (high transparency leads to lower average intensity), and shape (high curvature leads to lower average reflected intensity). The corresponding intensity maps at three different height ranges defined by equation \ref{eqn:3-layers} are shown, and the peak intensity neighborhood reflected from the glass appears in $I^t_{mid}$ due to our definitions of its range, and is highlighted by the yellow circle. Our formulation detects this pattern and linearly extrapolates the transparent obstacle's shape from it to safely navigate unknown environments. The green parallelograms show regions of interest which are defined by equation \ref{eqn:roi_definition}.}}
    \label{fig:gaussian}
    \vspace{-10pt}
\end{figure*}
\subsection{Light Intensity and Transparent Obstacles}
The intensity measures the amount of light energy that is reflected back to the lidar from any object in the environment. Typically, points reflected from diffuse surfaces (opaque, smooth, and planar/uncurved) satisfy $int \approx i_{max}$. This is because the opacity ensures that most of the light rays do not get transmitted through the object, and the smoothness and planarity ensure that they do not scatter away from the lidar. Further, $int$ depends on an object's proximity to the lidar, and the angle of incidence of the light. For transparent obstacles especially, the highest intensity is observed at $0^{\circ}$ incidence \cite{visagge}.

To detect transparent obstacles from point clouds and navigate, our approach uses multiple layers of 2D grid maps obtained by projecting the intensity of point clouds belonging to certain height intervals $H_j$ along the z-axis. The multiple layers together form a Multi-layer Intensity Map \cite{mim-arxiv}. In which, each 2D layer has dimensions $n \times n$, and $(n/2, n/2)$ denotes the robot/lidar's position w.r.t the map. Each grid is represented by its row and column coordinate $(r, c)$. Formally, a single layer at time $t$ is constructed from points belonging to an interval $H_j$, denoted as $I^t_{z \in H_j}$ is defined as, 

\begin{equation} \label{eqn:single-layer}
\begin{split}
    I^t_{z \in H_j}(r, c) &= \frac{\sum_{x}\sum_{y} int}{s^2} \\ 
    \forall \,\, x &\in [x_{low}, x_{low} + s], \,\, y \in [y_{low}, y_{low} + s], \\ 
    x_{low} &= \floor*{(r - \frac{n}{2}) \cdot s} \,\,\, \text{and} \,\,\, y_{low} = \floor*{(c - \frac{n}{2}) \cdot s} \\
    H_j &= [low_j, high_j], \,\, low_j \le high_j.
\end{split}
\end{equation}

Here, $s$ denotes the side length of the real-world square that each grid $(r, c)$ represents in $I^t_{z \in H_j}$ that is computed continuously for every time instant.  
\section{TOPGN: Transparent Obstacle Perception}
In this section, we explain how transparent obstacles are detected using the layers of the intensity maps. Next, we explain how we isolate a Transparent Obstacle Neighborhood (TON) in real-world scenarios, and linearly extrapolate the transparent obstacle shape for autonomously navigating in their presence. 


\subsection{Transparent Obstacle Intensity Distribution} \label{sec:gaussion-pattern} 

The intensities of point clouds incident on transparent obstacles along a horizontal plane as a function of the incident angle can be approximated as a Gaussian curve \cite{lasitha2022cartographer_glass,fusing-sonar}. Due to symmetry, this also holds true along the vertical/longitudinal plane as depicted in Fig. \ref{fig:gaussian}. The points with the peak intensity occur near the location where the angle of incidence is $\sim 0^{\circ}$ \cite{visagge} at height $h_{lid}$, and the intensity dissipates for points farther from this center. This Gaussian pattern of point cloud intensities is observed for all transparent obstacles taller than $h_{lid}$. To detect this pattern efficiently and use it for navigation, we use a three-layered intensity map $I^t_{3L} = [I^t_{low} | I^t_{mid} | I^t_{high}]$ that stacks three 2D grid maps ($I^t_{low}, I^t_{mid}, I^t_{high}$). Each of these layers is defined according to equation \ref{eqn:single-layer} by the limits in the $z$ axis specified as,
\begin{equation} \label{eqn:3-layers}
\begin{split}
    I^t_{low} &: z \in (0, h_{lid} - \Delta) \\
    I^t_{mid} &: z \in (h_{lid} - \Delta, h_{lid} + \Delta) \\
    I^t_{high} &: z \in (h_{lid} + \Delta, h_{lid} + 2\Delta ).
\end{split}
\end{equation}

Here, $\Delta$ is a height parameter that controls the number of points that are projected on to $I^t_{mid}$. It is chosen empirically such that all the points with the highest intensity lie within $h_{lid} \pm \Delta$ when the lidar is $d_{thresh}$ meters (threshold distance to maintain with obstacles) away from a completely transparent obstacle. Our definition of these layers leads to $I^t_{mid}$ registering a prominent region with high intensities while $I^t_{low}$, and $I^t_{high}$ register lower intensity values $\approx 0$ in the same grid position $(r, c)$ as shown in Fig. \ref{fig:gaussian}. We refer to any region that satisfies this condition as a \textit{Transparent Obstacle Region} (TON).

\subsection{Transparent Obstacle Isolation} \label{sec:isolation}
To isolate a TON from $I^t_{3L}$, we formulate the following condition,
\vspace{-10pt}
\begin{equation} \label{eqn:TO-condition}
    \mathcal{G}^t(r, c) = 
    \begin{cases}
    1 \,\,\, &\forall \{(r, c) \,\,\, s.t. \,\,\, ROI^t_{mid}(r, c) \in \mathbf{R}, \text{and} \\
    & ROI^t_{low}(r, c) < max(\mathbf{R})/3, \text{and} \\
    & ROI^t_{high}(r, c) < max(\mathbf{R})/3 \} \\
    0 \,\,\, &\text{Otherwise}.
    \end{cases}
\end{equation}

Here, $\mathbf{R}$ denotes a range of intensities, and $ROI^t_{mid}, ROI^t_{low}, ROI^t_{high}$ represent an $m \times m$ ($m < n$) regions of interest defined in the corresponding intensity maps as, 
\begin{multline}
    ROI^t_{low/mid/high} = \{I^t_{low/mid/high}(r, c) | \\
    r, c \in [\frac{n}{2} - \frac{m}{2},  \frac{n}{2} + \frac{m}{2}]\}.
    \label{eqn:roi_definition}
\end{multline}

The isolated ROIs contain minor artifacts due to noise which are filtered out. \blue{Our filtering approach is based on identifying the contours of all the regions with 1's, and removing the ones with low areas.} We use ROIs to further reduce computation costs. They are visually represented in green in Fig. \ref{fig:gaussian}. $\mathcal{G}^t$ is an $m \times m$ grid map that contains only the grids belonging to various TONs (that contain value 1) at any time instant $t$.

\subsection{Transparent Obstacle Extrapolation} \label{sec:extrapolation}
$\mathcal{G}^t$ only indicates the presence of a transparent obstacle (see Fig. \ref{fig:mapping}) at a time instant and does not represent its true shape, which is required to avoid collisions during navigation. Therefore, we propose a method to linearly extrapolate the transparent object based on the $j^{th}$ transparent obstacle neighborhood $TON_j$ in $\mathcal{G}^t$. To this end, we first compute the centroid grid for $TON_j$ as  $(r^j_{cen} , c^j_{cen}) = (\sum{r}/size(TON_j), \sum{c}/size(TON_j)) \,\, \forall \,\, r, c \,\, s.t. \,\,  \mathcal{G}^t(r, c) \\= 1$, where $size(TON_j)$ returns the number of grids in $TON_j$. Next, bounding circles $C^j_{bound}$ centered at $r^j_{cen}, c^j_{cen}$, and radius equal to the distance from the centroid to the farthest point in $TON_j$ as shown in Fig. \ref{fig:glass-extrapolation} are computed.   

Now, let us consider a light ray in 3D that is incident at $\sim 0^{\circ}$ on a transparent surface. It is by definition perpendicular to the tangent to the surface at the point of incidence (see Fig. \ref{fig:glass-extrapolation}a). The line equation of such a light ray in 2D, relative to $\mathcal{G}^t$ can be obtained by connecting the position of the lidar, and the centroid of a TON (see Fig. \ref{fig:glass-extrapolation}). Then, the vectors of the incident light ray, and the tangent line perpendicular to it can be represented as,
\begin{equation}
\begin{split}
    \mathbf{light}^j &= [r^j_{cen} - m/2, c^j_{cen} - m/2]^\top, \\
    \mathbf{tangent}^j &= [c^j_{cen} - m/2, -(r^j_{cen} - m/2)]^\top.
\end{split}
\end{equation}

Let the point of intersection of the $\mathbf{light}$ ray with the corresponding $C^j_{bound}$ be $(r^j_{int}, c^j_{int})$. We extrapolate the tangent line segment from $(r^j_{int} , c^j_{int})$ on either direction by the robot's radius $r_{rob}/s$ grids as, 
\begin{equation}
    E^j = \{ [r^j_{int}, c^j_{int}] \pm \frac{r_{rob}}{s} (\frac{\mathbf{tangent}^j}{\lVert \mathbf{tangent}^j \rVert}) \},
\end{equation} 


We extrapolate by the robot's radius on either side to minimize the amount of free space that is considered an obstacle by the robot's planner. Line segment $E^j$ is considered as a half-plane beyond which the robot should consider an obstacle region and avoid. This is depicted in Fig. \ref{fig:glass-extrapolation}. All grids corresponding to vectors and line equations are integerized. We omit this in the equations for readability. Finally, we obtain a grid map containing all the extrapolated TONs, and refer to it as $\mathcal{G}^t_{extrap}$. It is defined as $\mathcal{G}^t_{extrap}(r, c) = 1 \,\, \forall (r, c) \in E_j \,\, \forall j$. 
\begin{figure}[t]
    \centering
    \includegraphics[width=0.9\columnwidth,height=4.5cm]{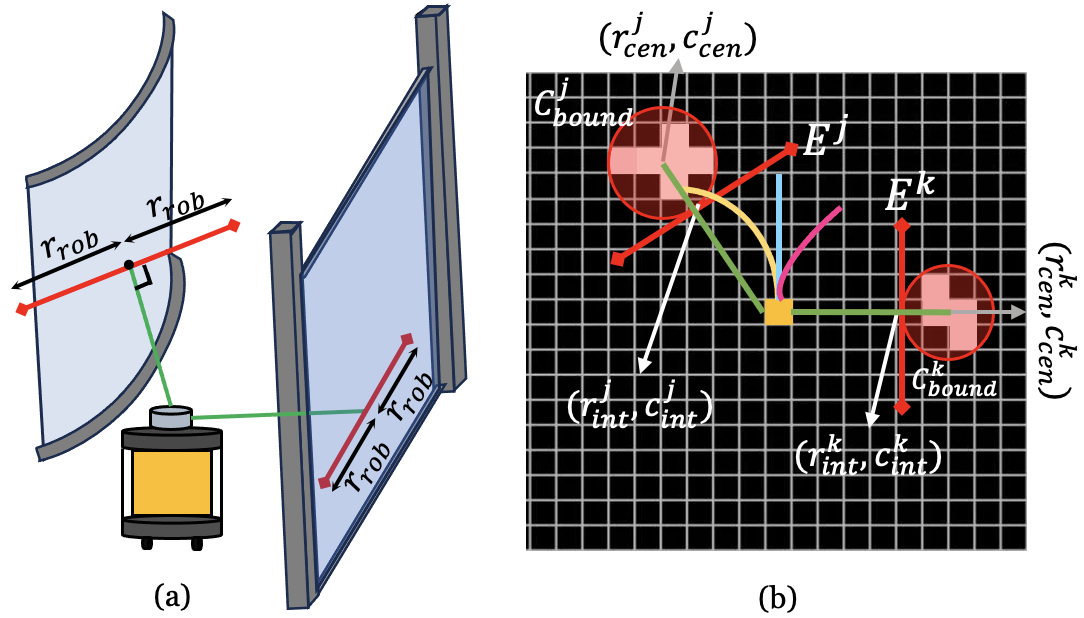}
    \caption{\small{(a) The green lines depict the light ray that is incident at $0^{\circ}$ on a transparent obstacle (in blue). The line segment perpendicular to this light ray is represented in red and is extended from the point of incidence on either side by the radius of the robot $r_{rob}$. (b) The same scenario on $\mathcal{G}^t$ where the light ray connects the robot (in yellow) with $(r^j_{cen}, c^j_{cen})$. The red lines are extended on either side by $r_{rob}/s$ grids in red. All grids on the other side of the red lines are considered as obstacles for time instant $t$. A few of the robot's instantaneous candidate trajectories when integrated with a planner \cite{fox1997dwa} are shown in blue, yellow, and pink. An optimal trajectory is chosen based on its distance away from obstacles, and the progress/heading towards the goal. In this scenario, the pink trajectory is preferred over the others as it is away from obstacles.}}
    \label{fig:glass-extrapolation}
    \vspace{-15pt}
\end{figure}

\subsection{Collision Avoidance}
Our transparent obstacle extrapolation can be integrated with any velocity/trajectory planning method \cite{fox1997dwa,time-elastic-band} that evaluates navigation \textit{costs} for the robot's candidate trajectories based on their proximity to obstacles, and the robot's goal. We use the work by Fox et al. \cite{fox1997dwa} along with our map layers, and $\mathcal{G}^t_{extrap}$ to navigate unknown environments with transparent obstacles. To first obtain a complete representation of all the obstacles (opaque and extrapolated transparent obstacles) in the environment, we obtain a grid map that can be used for navigation as $I^t_{nav}$,
\begin{equation} \label{eqn:nav-grid-map}
\begin{split}
    I^t_{nav} &= I^t_{low} + I^t_{mid} + I^t_{high} \\
    I^t_{nav}(B, B) &= I^t_{nav}(B, B) +  \mathcal{G}^t_{extrap} \\
    B &= \big(\frac{n}{2} - \frac{m}{2} : \frac{n}{2} + \frac{m}{2}\big). 
\end{split} 
\end{equation}

In equation \ref{eqn:nav-grid-map}, $I^t_{nav}(B, B)$ represents an $m \times m$ subset similar to the green ROI in Fig. \ref{fig:gaussian} where the extrapolated transparent obstacles are added. $I^t_{nav}$ represents \textit{all} environmental obstacles and can be used by \cite{fox1997dwa} to evaluate the costs for candidate trajectories and compute the least cost trajectory for the robot to follow. Please refer to the supplementary document for additional details on the planner.

\begin{figure}[t]
    \centering
    \includegraphics[width=\columnwidth,height=4.3cm]{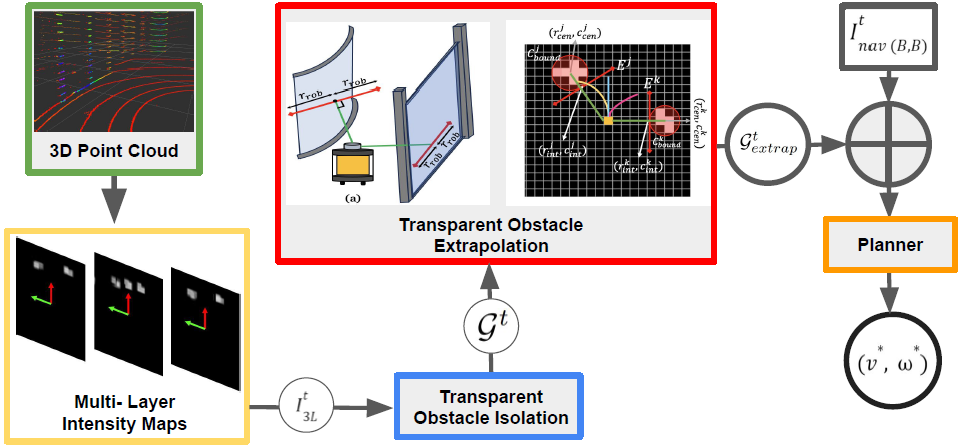}
    \caption{\small{TOPGN's overall system architecture. Three-layered intensity map $I^t_{3L}$ is extracted from the 3D lidar point cloud to isolate the transparent obstacles. The linear extrapolation is performed on isolated obstacles to estimate the true shape of the obstacles for collision avoidance. The extrapolated transparent obstacles are combined with the navigation cost map $I^t_{nav}$ to generate collision-free and goal-reaching actions from the planner. This overall framework demonstrates superior real-time transparent obstacle detection capabilities compared to state-of-the-art vision-based and lidar-based approaches. }}
    \label{fig:sys-arch}
    \vspace{-15pt}
\end{figure}

\begin{lemma}
    A candidate robot trajectory that does not intersect with any line segment $E^j$ during navigation at time instant $t$ guarantees collision avoidance with every transparent obstacle in the robot's vicinity at that instant. 
\end{lemma}  
\begin{proof}
    Let us consider an environment with transparent obstacles of which $K$ unoccluded 3-dimensional, smooth (i.e., their shape is a differentiable curve) transparent obstacles lie within the lidar's sensing range. There exist $K$ closest points on these obstacles from the lidar. The lines connecting the lidar and these closest points are normal to the transparent obstacle and will lead to corresponding transparent obstacle neighborhoods (TONs) in $\mathcal{G}^t$. Therefore, all the transparent obstacles satisfying smoothness, and 3D assumptions can be detected in $\mathcal{G}^t$ at time instant $t$. This ensures that there exists a corresponding tangent line segment $E^j$ that can be extrapolated for the $j^{th}$ transparent obstacle neighborhood in $\mathcal{G}^t$ as shown in Fig. \ref{fig:glass-extrapolation}b. We refer to the map with the extrapolated line segments as $\mathcal{G}^t_{extrap}$.

    By construction, the closest point(s) in the $j^{th}$ transparent obstacle are contained beyond $E^j$ when viewed from the lidar. Let $traj^I_{k}$ denote the $k^{th}$ candidate trajectory represented as a set of row and column coordinates relative to the intensity map.  If $traj^I_{k} \cap E^j = \emptyset \implies traj^I_{k} \cap TON_j = \emptyset$, guaranteeing collision avoidance with transparent obstacles at instant $t$.
\end{proof}

\blue{Our approach of isolating (section \ref{sec:isolation}) and extrapolating (section \ref{sec:extrapolation}) transparent obstacles allows a robot to avoid collisions with them in completely unknown environments without any prior mapping. Next, we highlight the generality of our TON isolation by demonstrating how it can also be used for mapping transparent obstacles in real-time.}

\subsection{Application to Mapping}
\blue{In this section, we explain how our transparent object isolation in various instances of $\mathcal{G}^t$ can be used for mapping the environment in 2D in a single walk-through.} To this end, the previous instances of the neighborhood ($\mathcal{G}^{t-1}, \mathcal{G}^{t-2}, ... , \mathcal{G}^{t-t_{past}}$) transformed relative to the present time instant $t$ are added to the current $I^t_{mid}$ to construct the true shape of a transparent obstacle. Here, $t_{past}$ is the number of past instances considered for mapping the transparent obstacle. To transform the position of obstacles in $\mathcal{G}^{t-k}$ relative to $I^{t}_{mid}$, we use transformation matrices $T^{t-k}_t$ computed based on the robot's motion between instances $t - k$ to $t$. That is,
\begin{equation} \label{eqn:mapping}
\begin{split}
    \mathcal{G}^{t-k}_t &= T^{t-k}_t \cdot \mathcal{G}^{t-k}, \\
    I^{t}_{mid}(B, B) &= I^{t}_{mid}(B, B) + \mathcal{G}^{t-k}_t \,\,\, \forall k \in [1, t_{past}] \\
\end{split} 
\end{equation}

This addition is depicted in Fig. \ref{fig:mapping}. Finally, to map all obstacles, we perform $I^t_{mapping} = I^t_{low} + I^t_{mid} + I^t_{high}$.

\begin{figure}[t]
    \centering
    \includegraphics[width=0.7\columnwidth,height=4.5cm]{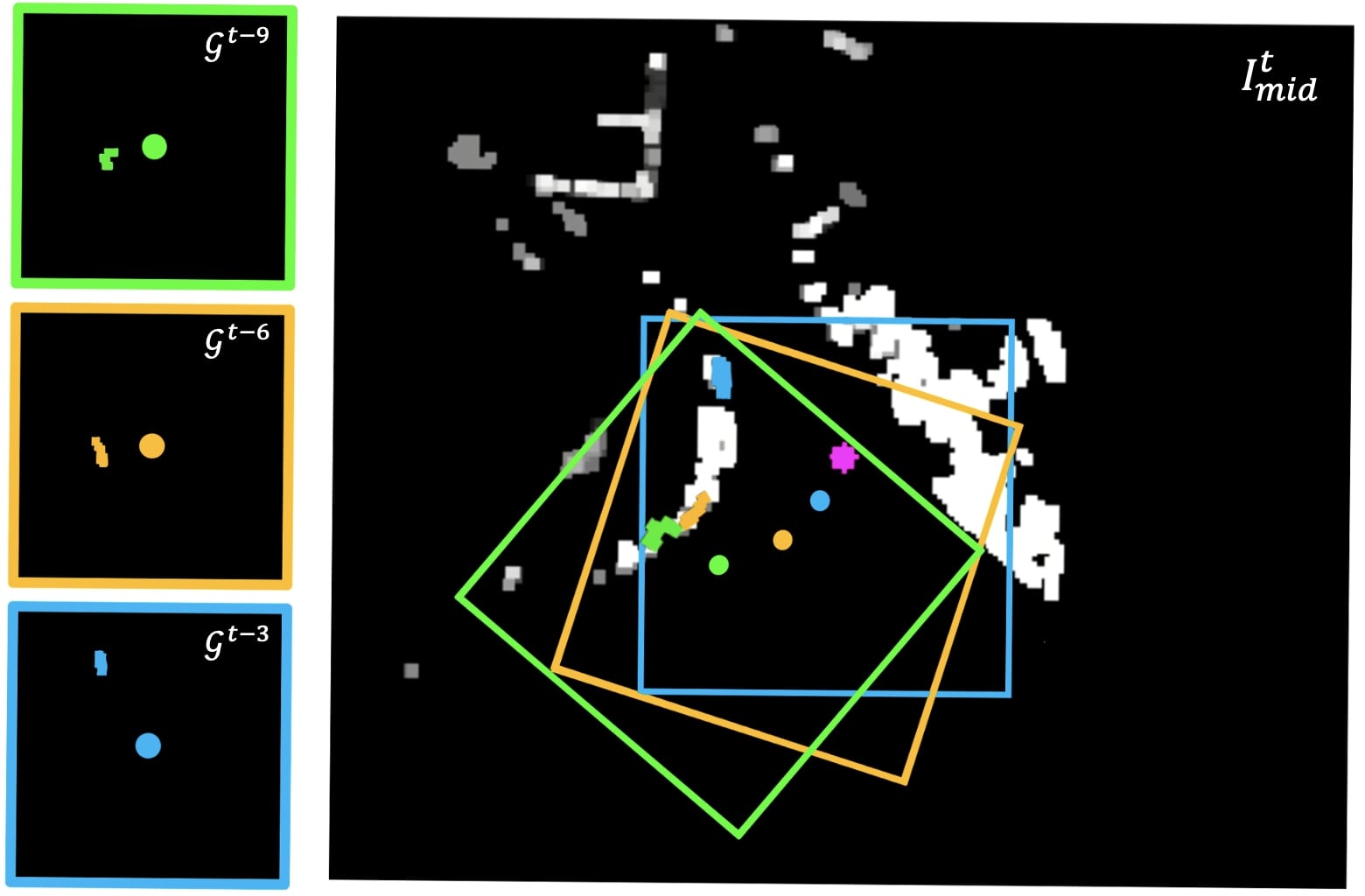}
    \caption{\small{$I^t_{mid}$ for the curved glass scenario shown in Fig. \ref{fig:gaussian} [center], and the transformed $\mathcal{G}^{t-3}$ (blue square), $\mathcal{G}^{t-6}$ (yellow square), and $\mathcal{G}^{t-9}$ (green square) added to it. The centers of these squares depict the robot's movement as time progresses. This addition reconstructs the true shape of the transparent obstacle for mapping.}}
    \label{fig:mapping}
    \vspace{-15pt}
\end{figure}

\section{Results and Experiments}
In this section, we explain TOPGN's implementation, and real-world experiments/evaluation, and demonstrate its advantages for autonomous navigation.  

\subsection{Hardware and Software Implementation}
We implement our proposed approach on a Turtlebot 2 robot equipped with a Velodyne VLP16 lidar, and a laptop with an Intel i7 CPU and NVIDIA RTX 3060 GPU. The robot is also equipped with an Intel Realsense d435 camera to collect images to compare with RGB segmentation methods. We use the following parameters in the implementation: $n = 200, m = 80, \mathbf{R} = [100,130], h_{rob} = h_{lid} = 0.5m, r_{rob} = 0.3m, \Delta = 0.2m$.

\subsection{Evaluations}
\textbf{Comparison Methods}: We compare TOPGN with two types of methods for glass, and transparent obstacle detection: 1. Semantic segmentation methods that use RGB images, and 2. SLAM methods that use lidar 2D scans or 3D point clouds. We use the following semantic segmentation methods: GDNet \cite{gdnet}, TransLab \cite{translab}, Mirror-Net \cite{mirrorNet}, RGB-T segmentation \cite{rgb-thermal}, and the following mapping methods: Gmapping \cite{grisetti2007gmapping}, Glass-SLAM \cite{glass-slam}, and Glass-Cartographer \cite{lasitha2022cartographer_glass}. Gmapping \cite{grisetti2007gmapping} is used as a baseline to indicate the challenges in detecting glass.  

\textbf{Test Dataset}: To evaluate both types of methods in a uniform manner, we treat transparent obstacle detection as a segmentation problem and create a test dataset with three sets of inputs and masks. The first set of masks outlines transparent obstacles on RGB images. The second set outlines them on local square segments cropped from the global 2D grid maps computed by Gmapping \cite{grisetti2007gmapping}, Glass-SLAM \cite{glass-slam}, and Glass-Cartographer \cite{lasitha2022cartographer_glass}. The third set marks transparent obstacles on local 2D grid maps to evaluate TOPGN. We use TOPGN to create a local map of obstacles (similar to Fig. \ref{fig:mapping}), and evaluate its detection capabilities based on that map. The test set is collected in various environments with sharp lighting changes, different levels of transparency, color, textures, and shapes in transparent obstacles (see Figs. \ref{fig:cover_image}, \ref{fig:various-env}, supplementary material).  

\begin{figure*}[t]
    \centering
    \includegraphics[width=\linewidth]{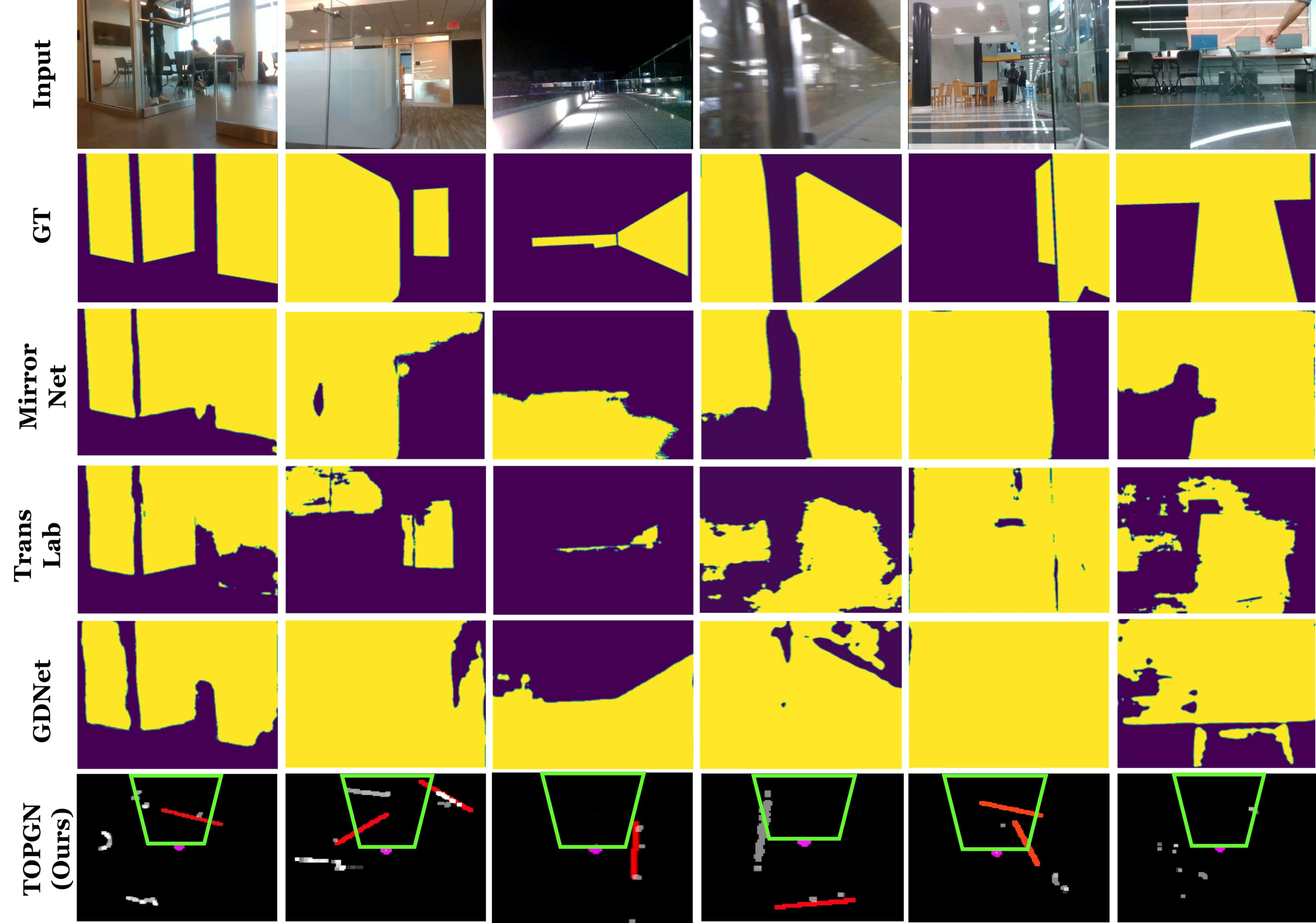}
    \caption{\small{\rev{The glass detection outputs (yellow indicates transparent obstacle) of MirrorNet \cite{mirrorNet}, TransLab \cite{translab}, and GDNet \cite{gdnet} in a few images from our test benchmarks, which contains scenarios with strong lighting changes (columns 1, 3), reflections (column 5), motion blur (column 4), and curved glass (columns 2, 5). In many instances, these segmentation methods wrongly classify free space as glass. Further, the robot's motion could cause blurring in some sets of frames, which leads to inaccurate detections.  During navigation, such errors cause the robot's planner to freeze or collide. TOPGN accurately extrapolates (red lines) the transparent obstacle in most scenarios. The pink circle denotes the robot's position and the green polygons represent the camera's field of view (FOV). We also depict some failure cases (columns 5 and 6) with highly non-convex and tilted transparent obstacles.}}}
    \label{fig:various-env}
    \vspace{-10pt}
\end{figure*}


\begin{figure}[t]
    \centering
    \includegraphics[width=\columnwidth]{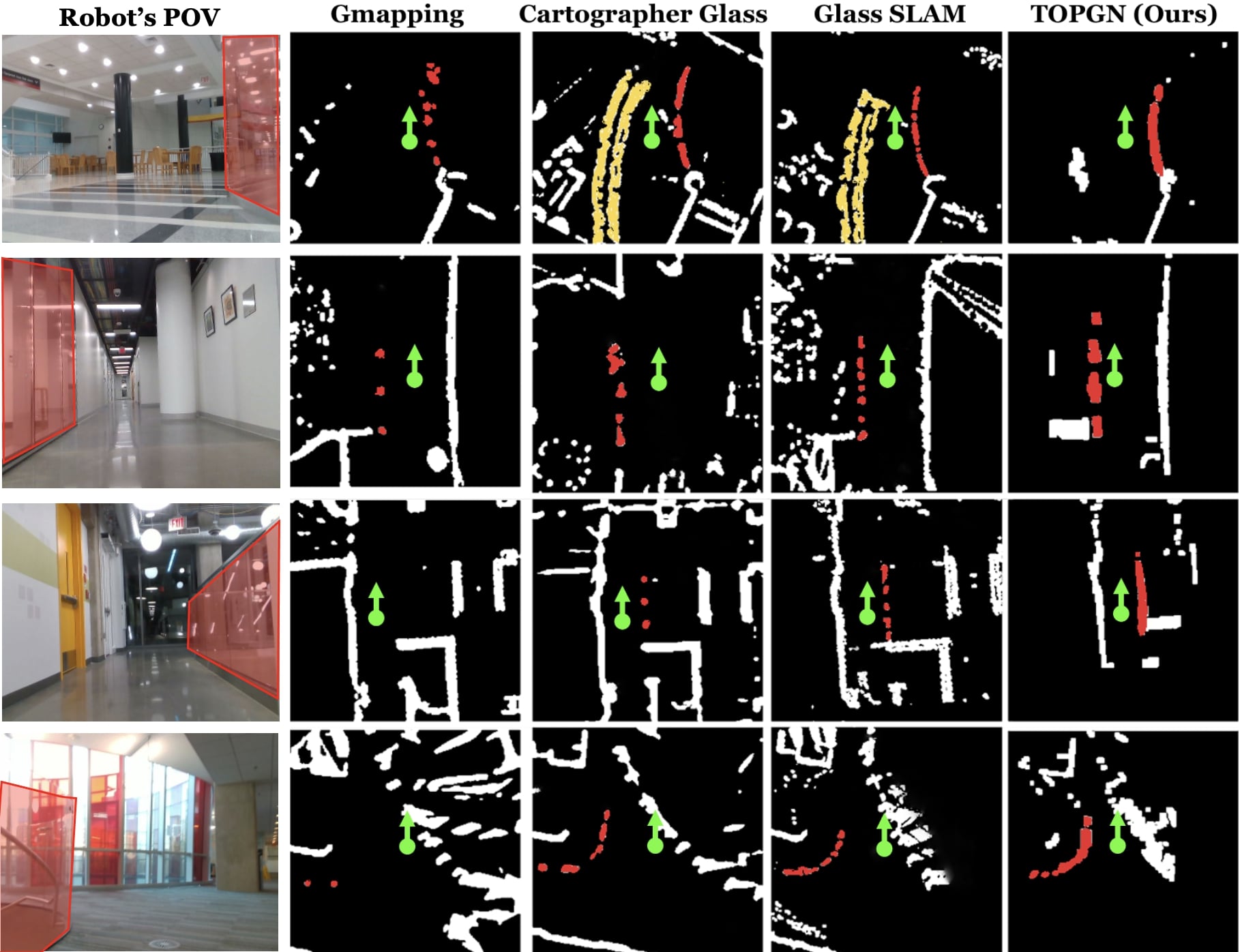}
    \caption{\small{Transparent obstacle mapping outputs from the lidar-based mapping methods: Gmapping  \cite{grisetti2007gmapping}; Cartographer Glass \cite{lasitha2022cartographer_glass} and Glass SLAM \cite{glass-slam} compared to TOPGN (ours). Red regions indicate the transparent objects detected from each method on a local map where the robot is located at the center (green color arrow). Artifacts due to erroneous glass detection from SLAM methods are marked in yellow color. The perspective from the robot as it maps the environment is shown in the first column for context, and the transparent obstacles are marked in red.}}
    \label{fig:map-comp}
\end{figure}

\textbf{Metrics}: To evaluate and compare various methods in terms of transparent obstacle detection, we use four widely adopted metrics: 
\begin{itemize}
\item Mean IoU (mIoU): It is the area of overlap between the segmentation output and the ground truth divided by the area of union between the predicted segmentation and the ground truth.

\item Pixel Accuracy (PA): It denotes the percent of pixels that are accurately classified in the image. It is calculated as,

\begin{equation}
    PA = \frac{\sum_{\forall j} P_{jj}}{\sum_{\forall j} T_j}.
\end{equation}

Here, $P_{jj}$ is the number of pixels predicted to be in class j, and belonging to class j, and $T_j$ denotes the total number of pixels labeled as class j. 

\item $F_1$ score: It is the harmonic mean of the average precision and average recall calculated as,

\begin{equation}
    F_1 = \frac{2 \cdot Precision \times Recall}{Precision + Recall}.
\end{equation}

\item Mean Absolute Error (MAE): It is a measure of errors between paired observations (predictions and ground truth). It is calculated as,

\begin{equation}
    MAE = \frac{\sum^{tot}_{j = 1}|y_j - x_j|}{tot}.
\end{equation}

Here, $tot$ represents the total number of data points (e.g. pixels), and $y_j, x_j$ represent the prediction and the ground truth of the data point respectively.
\end{itemize}

Additionally, to evaluate a method's applicability for real-time robot navigation, we also measure its inference rate. We define inference rate as the inverse of the time taken by a method to compute a segmentation mask or update a map based on local observations. 

\textbf{Navigation Evaluations}: To assess the benefits of TOPGN for robot navigation, we compare the navigation success rate when using: 1. vanilla planner using 2D lidar scans \cite{fox1997dwa}, 2. planner using the segmentation, 3. planner using SLAM methods, and 4. planner using our proposed extrapolation method (section \ref{sec:extrapolation}). The success rate is defined as the number of times the robot reaches its goal without colliding with any obstacle and freezing or halting forever in 10 trials. For the evaluation, we utilize the outputs (which contain transparent obstacle locations w.r.t the robot) of all these methods to evaluate the robot's trajectories using \cite{fox1997dwa}. The planner \cite{fox1997dwa} chooses trajectories that do not intersect with any transparent obstacle, and also utilizes 2D laser scans to detect and avoid opaque obstacles. 
We evaluate it in five scenarios with transparent obstacles: 1. Indoor setting with a glass door (one open, and one closed) and strong backlighting, 2. Indoor setting with reflective surfaces and strong multiple lights, 3. Outdoor setting with low light, 4. Indoor setting with a curved acrylic transparent obstacle, 5. Indoor lab setting with mirrors, arbitrarily shaped transparent PVC, and acrylic sheets.

\begin{figure}[t]
    \centering
    \includegraphics[width=\columnwidth]{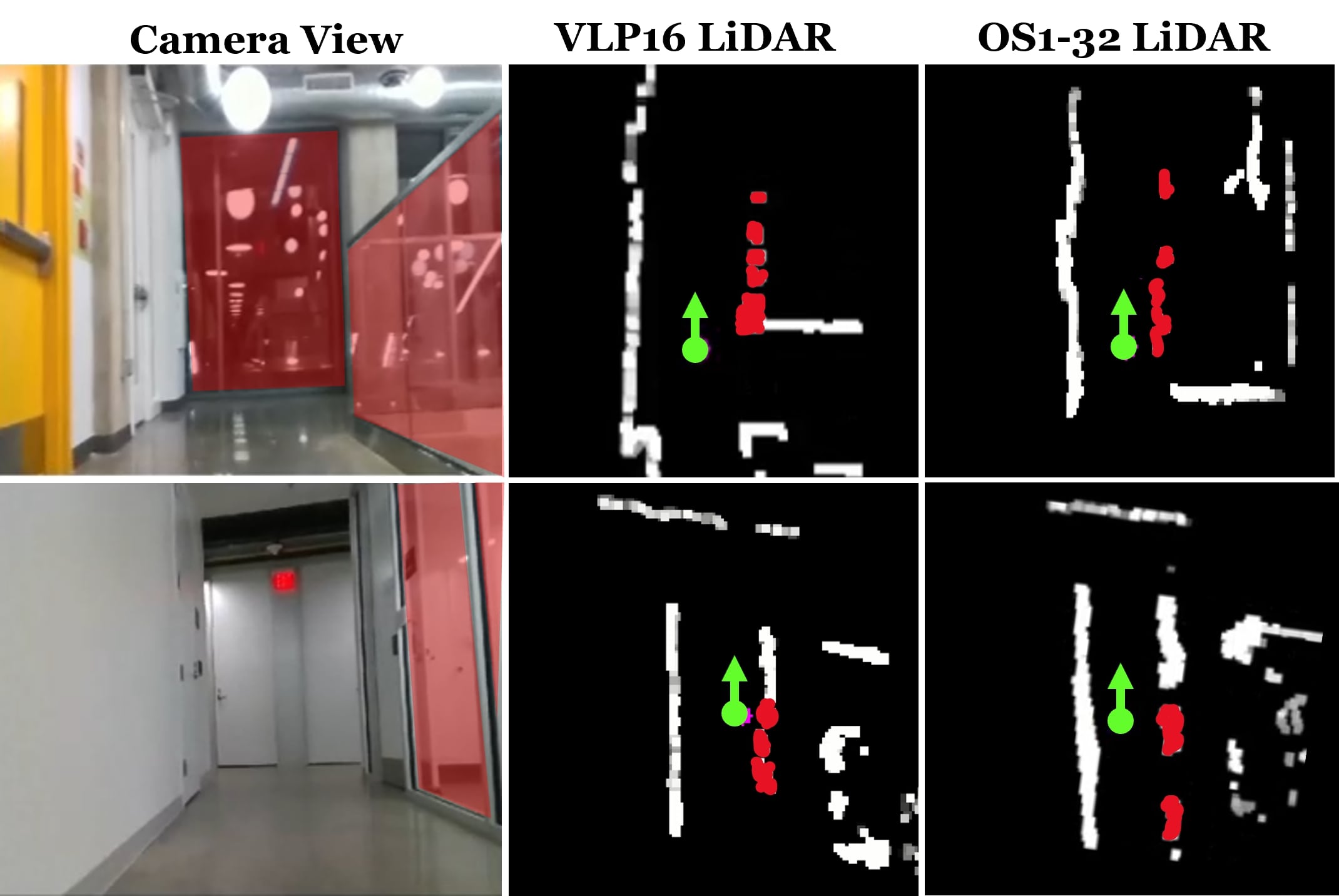}
    \caption{\small{\rev{TOPGN's transparent obstacle detection performance comparison for two different 3D LiDARs with different channel resolution: 1. Velodyne VLP16 has 16 verticle channels; 2. OS1-32 Lidar has 32 verticle channels. We observe that our method demonstrates comparable detection performance for different 3D LiDAR sensors that have different channel resolutions. The robot's camera views are presented on the left to help understand the transparent regions (marked in red) on the cost maps.} }}
    \label{fig:lidar-comparison}
\end{figure}

\begin{figure}[t]
    \centering
    \includegraphics[width=0.75\columnwidth]{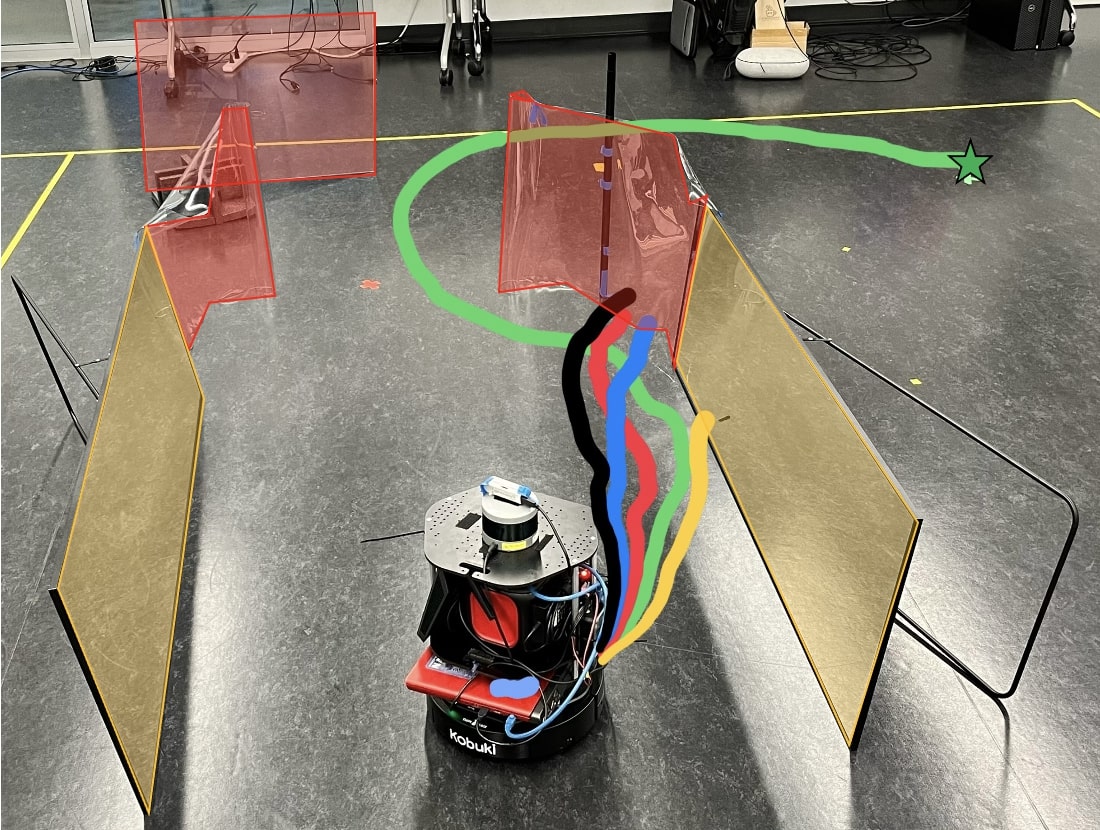}
    \caption{\small{Scenario 5 that requires the robot to navigate in the presence of mirrors (highlighted in yellow), arbitrarily-shaped transparent PVC and acrylic sheets (highlighted in red). We observe that TOPGN (green) navigates the robot by avoiding all transparent obstacles. Semantic segmentation methods (in blue \cite{gdnet} and red \cite{translab}), and using 2D lidar scans \cite{fox1997dwa} cannot detect transparent PVC, leading to collisions. Glass-SLAM \cite{glass-slam} (in yellow), due to slow map update rates, directly collides with the mirror.}}
    \label{fig:lab-scenario}
\end{figure}

\begin{table}[t]
\centering
\resizebox{\columnwidth}{!}{
 \begin{tabular}{l | c | c | c | c | c} 
 \hline
 Methods & MIoU $\uparrow$ & PA $\uparrow$ & $F_1\uparrow$ & MAE $\downarrow$ & Infer. Rate$\uparrow$ \\ [0.5ex] 
 \hline\hline
 TransLab \cite{translab}  & 0.564 & 0.809 & 0.554 & 0.190 & 1.61  \\ 
 Mirror-Net \cite{mirrorNet} & 0.412 & 0.631 & 0.701 & 0.369 & 0.40 \\
 GDNet \cite{gdnet} & 0.502 & 0.659 & 0.627 & 0.341 & 9.32  \\
 RGBT Segmnt. \cite{rgbt-glass-segmentation} & 0.215 & 0.724 & 0.385 & 0.275 & 0.45  \\
 \hline\hline
 Gmapping \cite{grisetti2007gmapping}                       & 0.497 & 0.623 & 0.484 & 0.368 & 0.23\\
 Glass-SLAM \cite{glass-slam}                               & 0.841 & 0.976 & 0.796 & 0.013 & 0.36 \\ 
 Cartographer Glass \cite{lasitha2022cartographer_glass}    & 0.813 & 0.965 & 0.824 & 0.025 & 7.34 \\ 
\hline\hline
 TOPGN w VLP16 Lidar (ours) & \textbf{0.872} & \textbf{0.992} & \textbf{0.929} & \textbf{0.008} & \textbf{45.67} (on CPU)\\ [1ex] 
  \rev{TOPGN w OS1-32 Lidar (ours)} & \rev{\textbf{0.881}} & \rev{0.928} & \rev{0.893} & \rev{0.011} & \rev{38.24 (on CPU)}\\ [1ex]
 \hline
 \end{tabular} }
 \caption{\small{Transparent obstacle detection performance comparisons for semantic segmentation and lidar-based SLAM methods against TOPGN (Ours) using various metrics.   }
}
\label{tab:perception-comparison}
\vspace{-10pt}
\end{table}

\begin{table}[t]
\resizebox{\columnwidth}{!}{
\begin{tabular}{c c c c c c c} 
\toprule
\textbf{Metrics} & \textbf{Method} &  \textbf{Scn. 1} & \textbf{Scn. 2} & \textbf{Scn. 3} & \textbf{Scn. 4} &  \textbf{Scn. 5}\\  
[0.5ex] 
\hline
\multirow{6}{*}{\rotatebox[origin=c]{0}{\makecell{\textbf{Success Rate (\%) $\uparrow$ }}}}
 & DWA Planner \cite{fox1997dwa}            & 0 & 0 & 20 & 0 & 0 \\
 & TransLab \cite{translab}             & 20 & 50 & 0 & 30 & 0\\
 & GDNet \cite{gdnet}                   & 40 & 20 & 0 & 20 & 0\\
 & Gmapping \cite{grisetti2007gmapping} & 0 & 10 & 20 & 0 & 0\\
& Glass-SLAM \cite{glass-slam}          & 10 & 30 & 40 & 30 & 0\\
 & TOPGN (ours)                         & \textbf{70} & \textbf{100} & \textbf{100} & \textbf{80} &  \textbf{60}\\
\bottomrule
\end{tabular}
}
\caption{\small{Navigation performance comparison for four scenarios that include transparent obstacles with different sizes, shapes, and under challenging lighting conditions.}}
\label{tab:navigation-comparison}
\end{table}

\subsection{Analysis}
Table \ref{tab:perception-comparison} shows the transparent obstacle detection performance of the segmentation-based, and SLAM-based methods. \blue{We choose to compare the performance of methods that use two different modalities to highlight the severe limitations of image-based segmentation.} Of the segmentation methods, TransLab \cite{translab} performs the best, followed by GDNet \cite{gdnet}, MirrorNet \cite{mirrorNet}, and RGB-T segmentation's \cite{rgb-thermal} RGB-only model. All segmentation methods perform well in environments that are well-lit, provide ample context to indicate the presence of transparent obstacles, and contain planar/uncurved glass. When these conditions are not satisfied, their performance tends to deteriorate. Such scenarios are shown in Fig. \ref{fig:various-env} with the input RGB image, the corresponding ground truth (GT), and the outputs of three segmentation methods, and our method's linear extrapolation results in these cases. We observe that in many instances these segmentation methods incorrectly mark free space as a transparent obstacle (in yellow). This predominantly occurs in the parts of the input image that have bright lights or reflections from shiny surfaces.

Such methods also struggle in low-light conditions (Fig. \ref{fig:various-env} columns 3 and 4), and frames with motion blur (Fig. \ref{fig:various-env} column 4), which could occur in images captured from a robot. In certain cases with glass doors (when one is open, and the other is closed), these methods predict the free space to also contain glass (see Fig. \ref{fig:various-env} column 1) similar to the observations in \cite{rich-context-aggregation}.

TOPGN, on the other hand, accurately extrapolates transparent obstacles linearly in the scenes in columns 1-3. For column 1, TOPGN extrapolates the closed door accurately. Although the line extended to cover the free space in the figure, as the robot moves and the line's position and orientation change to open up the free space in most cases during navigation. For curved glass (columns 2 and 5), the line is extrapolated tangential to it. The scenario in column 5 is especially challenging as it contains two glass components: a curved glass wall, and an open door. TOPGN accurately extrapolates the curved wall. We observe that TOPGN does not extrapolate the glass in column 4. This is because it is already detected as an obstacle (in grey) due to the strong reflections from the dust settled on the glass. Lidar-based detection benefits from such real-world phenomena, and can detect obstacles regardless of the robot's motion. We discuss the scenario in columns 5 and 6 further under failure cases.   

SLAM-based methods \cite{glass-slam,lasitha2022cartographer_glass} are not affected by low-light, bright reflections, or motion blur, and can detect transparent obstacles accurately as presented in Fig. \ref{fig:map-comp}. Gmapping, which is not formulated to detect transparent obstacles, exhibits low detection accuracy (see Fig. \ref{fig:map-comp} column 2). Its accuracy springs from detecting some portions of glass based on the opaque railings around it. It is used for comparison to highlight the difficulty in detecting transparent obstacles in general. Glass-SLAM \cite{glass-slam}, and Cartographer Glass \cite{lasitha2022cartographer_glass} accurately detect glass of various shapes and conditions. However, the presence of dynamic obstacles could cause artifacts such as a trail of their positions to be recorded and marked as obstacles on the map as shown from yellow regions in row 1 of Fig. \ref{fig:map-comp}. This occurs because they are configured to record \textit{all} reflected points (of various intensities) to detect glass. Configurations that prevent such artifacts lead to poor glass detection. TOPGN's formulation, when applied for mapping robustly detects transparent obstacles in all these cases and is not affected by environmental conditions. Importantly, TOPGN's TON isolation based on the condition in equation \ref{eqn:TO-condition} ensures that other obstacles (e.g. dynamic pedestrians) are not detected as transparent obstacles. This is because the intensity condition in equation \ref{eqn:TO-condition} cannot be satisfied by opaque obstacles such as humans. For collision avoidance in unknown/unmapped environments, equation \ref{eqn:nav-grid-map} ensures that no other obstacle is missed as it combines the extrapolated glass with all the obstacles in $I^t_{low}, I^t_{mid},$ and $I^t_{high}$. For mapping, equation \ref{eqn:mapping} ensures that all obstacles are detected at each instant, and artifacts from dynamic obstacles are not added to the map. Equation 8 combines the middle-intensity map ($I^t_{mid}$ that contains all opaque obstacles) with a transformed $\mathcal{G}^{t-k}$ (which contains an instance of a transparent obstacle).

\textbf{Inference Rate}: For real-world implementation, perception methods must possess a high inference rate. Comparing the inference rates of these methods, only GDNet \cite{gdnet} executes $\sim 9$ Hz, making it suitable for real-time navigation. All other segmentation methods have a high computational overhead and do not execute in real-time on a mobile GPU. Glass-SLAM requires $\sim 3-4$ seconds to update its map, making it unsuitable for real-time navigation. Cartographer Glass \cite{lasitha2022cartographer_glass} maps faster, but in some cases may not update the map in time to avoid obstacles. TOPGN has a superior inference rate for both extrapolation and mapping and executes at $\sim 50$ Hz on a mobile laptop CPU, enabling real-time navigation.

\begin{figure}[t]
    \centering
    \includegraphics[width=\columnwidth]{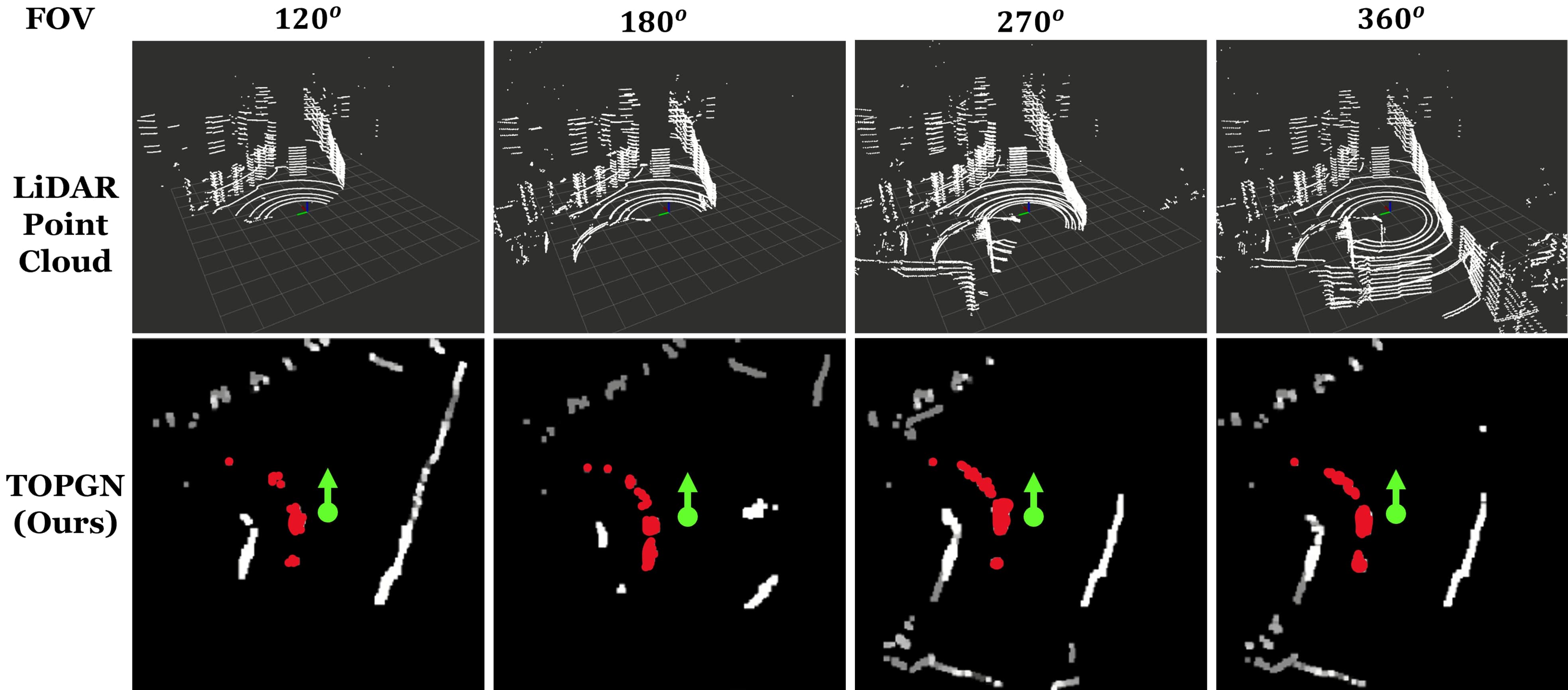}
    \caption{\small{\rev{TOPGN's performance comparison for different 3D lidar horizontal field of views (FOV). We observe that TOPGN's transparent obstacle detection performance is almost consistent until the lidar's horizontal FOV reduces from $360^{\circ}$ to $180^{\circ}$. However, TOPGN's performance degrades when the FOV is less than $180^{\circ}$. Red regions indicate the transparent obstacle regions detected by TOPGN.} }}
    \label{fig:fov-comparison}
\end{figure}

\textbf{Navigation Success Rate}: In terms of navigation success rate, we observe that using TOPGN's extrapolation greatly improves the robot's rate of reaching the goal. The planner when only using 2D laser scans \cite{fox1997dwa} to navigate and Gmapping, could only detect opaque obstacles, which led to collisions with transparent objects in all the trials in scenarios 1, 2, and 4. The planner when augmented with segmentation methods was able to successfully reach its goal in some trials. However, in most trials, incorrectly classifying free space as glass led to freezing issues, where the robot oscillates or halts indefinitely. We chose to evaluate TransLab \cite{translab}, and Glass-SLAM \cite{glass-slam} despite their low inference rates due to their reasonably high detection accuracy. However, their low inference rate led to severe oscillations when used with the planner. In scenario 1 in Fig. \ref{fig:cover_image}, Translab mistook the open door to be closed causing the robot to freeze before entering the room. GDNet, on the other hand, was more accurate in this scenario (also depicted in Fig. \ref{fig:various-env} column 1, row 5). Both TransLab \cite{translab}, and GDNet \cite{gdnet} experienced performance degradation in scenarios 2 and 3 in Fig. \ref{fig:cover_image} due to bright and low light conditions, respectively. In scenario 4, although the RGB-based methods managed to detect glass to a certain extent, the low field-of-view of the camera led to them not viewing the curved glass in many trials. Additionally, misclassifying free space as glass (see Fig. \ref{fig:various-env} column 2) also caused freezing.   

TOPGN, in all these scenarios, was able to isolate the TONs in the vicinity, and linearly extrapolate the transparent obstacle. Since the grids beyond the extrapolated line are considered as obstacles, the planner chose trajectories avoiding this region and averted collisions. Due to the narrow passages, and high curvature of glass in scenarios 1 and 4 respectively, our method led to collisions in some cases. In scenario 1 especially (also depicted in Fig. \ref{fig:various-env} column 1), TOPGN's linear extrapolation caused the robot to freeze in some cases due to the narrow passage. However, in most cases, the extrapolation aided in avoiding the closed door and reaching the goal.  

The results of our experiments in scenario 5 are shown in Fig.\ref{fig:lab-scenario}. Our method (in green) is able to navigate around the mirrors, and the transparent PVC. Segmentation methods such as GDNet \cite{gdnet} (in blue), and TransLab \cite{translab} (in red) cannot accurately detect the transparent PVC and collide with it. The reflections from the floor also confuse the segmentation and navigation. In many trials, these methods lead to freezing as most of the image is classified as a transparent obstacle. Glass-SLAM \cite{glass-slam} (in yellow) directly collides with the mirror due to its slow map update rate, and using 2D laser scans (in black) avoids the mirror but collides with the PVC similar to segmentation methods. This scenario highlights our method's linear extrapolation's capabilities in the presence of transparent obstacles with various shapes and materials.

\rev{\textbf{Compatibility with Different Lidar Sensors:} We evaluate our method's transparent obstacle detection capabilities for two different 3D lidar sensors in Fig. \ref{fig:lidar-comparison} and Table \ref{tab:perception-comparison}. Even though our experiments are conducted using a VLP16 Velodyne lidar with 16-channel vertical resolution, our methods demonstrate similar detection performance with an Ouster OS1-32 lidar (32-channel resolution) without any changes to the algorithmic parameters or threshold values. However, the 32-channel Ouster lidar results in a relatively lower inference rate due to the processing of a significantly higher dimensional point cloud compared to the 16-channel Velodyne lidar.} 

\rev{Our method cannot be used with any 2D lidar sensors since it requires multi-level intensities to isolate the transparent object regions. Moreover, 3D lidars with significantly low vertical resolution (e.g., less than 6-8 channels) might not be suitable for our approach due to the requirement of reasonable point clouds in each layer of the multi-layer intensity maps. However, we noticed that commonly available 3D lidars have at least 16 vertical channels.  }

\rev{\textbf{Effect of the Lidar's Horizontal FOV:} We observe that TOPGN's transparent obstacle detection demonstrates consistent qualitative performance for lidar point clouds captures from $360^{\circ}$ to $180^{\circ}$ sensor Field-of-Views (FOV) inf Fig. \ref{fig:fov-comparison}. However, the performance degrades beyond $180^{\circ}$ FOV. Hence, TOPGN can be used with relatively low FOV lidar sensors instead of the $360^{\circ}$ lidars. }

\textbf{Failure Cases:} We observe that in cases where a transparent obstacle has a sharp non-convex shape (Fig. \ref{fig:various-env} column 5 where the glass door is extrapolated incorrectly), or thin (almost 2-dimensional when viewed from the lidar) the lidar may not consistently be able to detect points that could avoid collisions at a future time instant. In such cases, the robot could get stuck or collide. However, we note that all existing methods also fail in such cases (e.g. see column 5 in Fig. \ref{fig:various-env}). Additionally, when transparent objects are highly inclined, the reflected laser points may not lie in $I^t_{mid}$ for it to be extrapolated.  We depict this scenario in Fig. \ref{fig:various-env} column 6. We observe that the detection depends on the angle of inclination, material and thickness of the transparent object.  

In our proof, the small TON is enclosed by a circle and the tangent line is constructed. This approximates the transparent object as convex at the TON. This holds true in real-world transparent obstacles and the way they are constructed. In extreme corner cases, our way of linear extrapolation may not be enough, and a higher-order curve could be more suitable. In cases with curved glass, we observe some gaps between the TONs detected in subsequent time steps. However, these gaps can be closed up by inflating the blobs/neighborhoods based on the robot's radius.

\section{Conclusions, Limitations, and Future Work}
We present TOPGN, a method to isolate instances of transparent obstacle neighborhoods in a multi-layer grid map representation containing point cloud intensities. Our method then extrapolates the regions that could potentially contain transparent obstacles, which when integrated with a navigation scheme, leads to superior rates of successfully reaching the robot's goal. We showed how TOPGN's transparent obstacle isolation can be useful in mapping applications. Our method is unaffected by adverse conditions such as harsh environmental lighting, reflections, motion blur from the robot, etc. 

Our method has a few limitations. In general, lidars have a circular blind spot around them and could miss obstacles closer than $1.5$m. Unlike RGB-based methods, lidar-based methods require several time instances to detect/map the true shape of transparent obstacles. We observe that in scenarios with sharply curved or tilted glass (e.g. curved glass with an open glass door), we may not be able to obtain sufficient TONs to detect and extrapolate the obstacle consistently. Our method is dependent on the specifications and quality of the lidar sensor. Therefore, using a lidar with a low vertical FOV would affect the glass detection accuracy. \blue{The transparent obstacles are implicitly assumed to be smooth or well approximated by the linear extrapolation. Similar to prior mapping methods that use lidar, our approach can be affected by odometry/localization errors when mapping transparent obstacles. This could lead to the robot freezing/halting indefinitely in narrow passages.}

In the future, we would like to address these limitations and investigate methods to obtain the same capabilities using a low-FOV lidar or time-of-flight sensor.

\bibliographystyle{IEEEtran}
\bibliography{References}

\clearpage
\section{Appendix}
In this document we provide additional definitions, and explanations that support and enhance the main manuscript.

\subsection{Integration with Planning}
We provide details on how a motion planner is integrated with our final navigation intensity map $I^t_{nav}$ (equation \ref{eqn:nav-grid-map}). At any time instant $t$, the values in the grids of $I^t_{nav}$ represent free space, opaque, or linearly extrapolated transparent obstacles. Therefore, $I^t_{nav}$ can be regarded as a \textit{cost map} containing the navigability costs (zero cost for free space, high positive costs for obstacles) of the robot's surroundings. Therefore, a motion planner could calculate costs for potential/candidate robot trajectories using $I^t_{nav}$, and select the trajectory with the lowest cost for execution by the robot. 

\subsubsection{Background on Motion Planners}
We integrate TOPGN with the Dynamic Window Approach (DWA)~\cite{fox1997dwa} to perform real-time navigation. In DWA, the robot's actions are represented as linear and angular velocity pairs $(v,\omega)$. $V_s = [[0, v_{max}], [-\omega_{max}, \omega_{max}]]$ is defined to be the space of all the possible robot velocities based on the maximum velocity limits $v_{max}$, and $\omega_{max}$. DWA formulates the following constrained velocity sets to compute dynamically feasible (i.e. executable by the robot) and collision-free velocities: (1) $V_d$, the dynamic window set contains the reachable velocities during the next $\Delta t$ time interval based on the robot's acceleration constraints;  (2) $V_a$, the admissible velocity space includes the collision-free velocities. The optimal velocity pair $(v^*,\omega^*)$ is then selected from the resulting velocity space $V_r = V_s \cap V_d \cap V_a$ by minimizing the following objective function:

\begin{equation}
\label{eq:dwa_obj_func}
    Q(v,\omega) = \sigma\big(\gamma_1 . head(.) + \gamma_2 . obs(.) + \gamma_3 . vel(.) \big).
\end{equation}

\begin{figure}[h]
    \centering
    \includegraphics[width=\columnwidth,height=7cm]{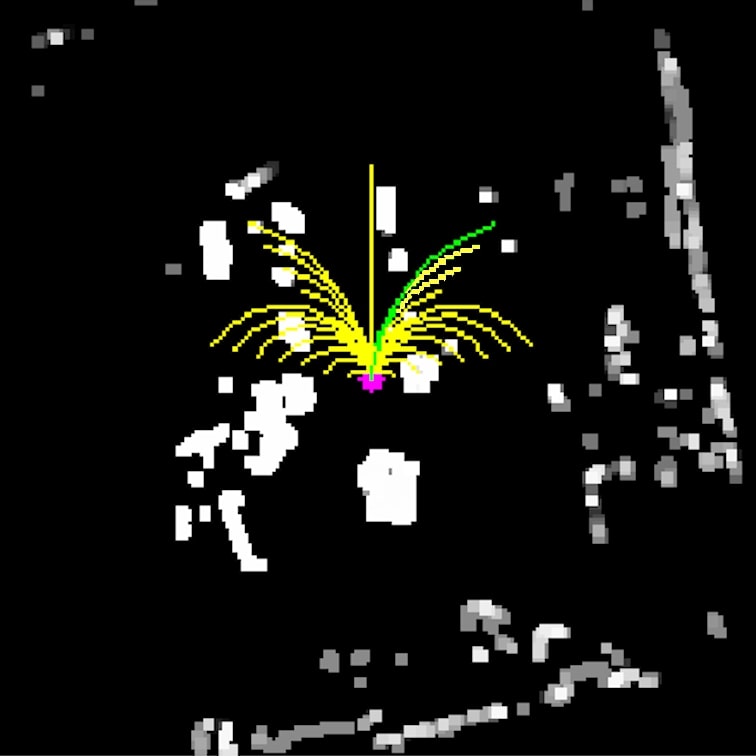}
    \caption{\small{Candidate trajectories (in yellow) shown relative to $I^t_{nav}$. The trajectory in green has the least-cost and is executed by the robot. The black regions represent free space, and grey/white regions represent opaque, and transparent obstacles.}}
    \label{fig:planning-integration}
\end{figure}

where $head(.)$, $obs(.)$, and $vel(.)$ are the cost functions \cite{fox1997dwa} to quantify a velocity pair's heading towards the goal, distance to the closest obstacle in the trajectory, and the forward velocity of the robot, respectively. $\sigma$ is a smoothing function and $\gamma_i, (i=1,2,3)$ are adjustable weights.

\subsubsection{Trajectory Cost Calculation}
For our integration, the $obs(.)$ cost is calculated by extrapolating the trajectories that each $(v, \omega)$ pair in $V_r$ would lead to within a time horizon $t_{hor}$. These trajectories are then transformed w.r.t $I^t_{nav}$ using the transformation between real-world dimensions and grid locations as mentioned in equation \ref{eqn:single-layer}. Then, the obstacle cost for the $k^{th}$ trajectory can be calculated as,

\begin{equation}
    \begin{split}
        traj^{I}_k = [(r_{1,k}, c_{1,k}), ... , (r_{j,k}, c_{j,k}), ... , (r_{lim,k}, c_{lim,k})] \\
        obs(traj^{I}_k) = \frac{1}{min(dist(\mathbf{O}^t, (r_{j,k}, c_{j,k})))}, \,\, j = \{1, 2, ..., lim\}.
    \end{split}
\end{equation}

Here, $\mathbf{O}^t$ is the set of obstacles (all grids that are not black in Fig. \ref{fig:planning-integration}) in $I^t_{nav}$. This cost calculation is repeated for every trajectory corresponding to a $(v, \omega) \in V_r$. The superimposed trajectories on $I^t_{nav}$ is shown in Fig.\ref{fig:planning-integration}. 

For integrating the planner with SLAM \cite{glass-slam}, the trajectories are obtained w.r.t the map generated by the SLAM method. For integrating with segmentation methods \cite{gdnet,translab}, the trajectories are transformed w.r.t the segmentation output image, and transparent obstacle costs are calculated. This is combined with the costs calculated by the planner \cite{fox1997dwa}.

\rev{\subsection{Parameters and default values} }

\rev{We summarize the details of the important parameters used in our method. We demonstrate that \textbf{without changing any of the parameters} our method can be used with different 3D lidar sensors that have different vertical channel resolutions (See comparison between Velodyne VLP16 and Ouster OS1-32 in Fig. \ref{fig:lidar-comparison}).  However, we would like to highlight that the performance of our method for a given lidar sensor can be improved if certain parameters are tuned accordingly. Hence, we discuss the effect of the parameters and threshold values used in our approach in the table below.}

\rev{Since our method's formulation depends on multi-layer intensity maps, it's mandatory to have enough vertical resolution in the point cloud to obtain at least 3 intensity layers. However, most commercially available 3D lidar sensors have at least 16 verticle channels or more which makes our algorithm compatible for general use. We further observed that even with enough vertical resolution, placing the lidar at a very low height can lead to difficulties in obtaining multi-layer intensities below the sensor's height level. Such issues can be mitigated simply by placing the sensor at a higher position. Hence, our overall approach can be easily deployed with any robot equipped with a 3D lidar. }

\begin{table*}[t]
\centering
\normalsize
\begin{tabular}{|p{1.5cm}|p{3cm}|p{0.75cm}|p{1.5cm}|p{8cm}|} 
\hline

\blue{\textbf{Parameter}} & \blue{\textbf{Description}} & \blue{\textbf{Range}} & \blue{\textbf{Default Value}} & \blue{\textbf{Notes}}   \\ [0.5ex] 
\hline
\blue{$n$}  & \blue{2D Intensity map width/height} & \blue{$\mathbb{Z}^+$} & \blue{200} &  \blue{Size $n$ of the intensity map indicates the sensing range of the lidar that we are interested in ($\sim$ 5-meter radius around the robot in our case). For a fixed sensing range in the real world, larger intensity maps lead to better spatial resolution in the real world (denoted by $s$). Extremely low-resolution intensity maps can close the narrow passages between the objects which makes it difficult to navigate a robot.} \\
\hline
\blue{$s$}  & \blue{Side length of the real-world square corresponding to a grid in a 2D map} & \blue{$\mathbb{Z}^+$} & \blue{0.05 meters} &  \blue{Smaller $s$ values indicate that the intensity map has a higher spatial resolution and vice versa. The effect of this parameter is correlated with the size of the intensity map.}\\
\hline
\blue{$h_{lid}$}  & \blue{Height of the lidar sensor location from the ground} & \blue{$\mathbb{Z}^+$} & \blue{0.48 meters} &  \blue{Height of the lidar sensor affects the TOPGN's performance only when the height is extremely low where even one additional intensity map layer cannot be defined below the sensor height level. Because our formulation requires at least one intensity map layer below the mid-intensity map (which is at the lidar's height level) to identify the TONs.} \\
\hline
\blue{$\Delta$}  & \blue{Height parameter that controls the height range considered for each intensity map layer} & \blue{$\mathbb{Z}^+$} & \blue{0.2 meters} &  \blue{Empirically chosen such that all the points with the highest intensity lie within $h_{lid} \pm \Delta$ when the lidar is $d_{thresh}$ meters away from a completely transparent obstacle (threshold distance to maintain with obstacles). $d_{thresh}$ is a robot-dependent parameter that can be obtained from the definition below. We believe that this is a simple calibration step one can perform if the hardware setup is significantly different from ours. Otherwise, our method will perform comparably with any similar robot setup without any changes to the parameters.} \\
\hline
\blue{$d_{thresh}$}  & \blue{Threshold distance to maintain with a transparent obstacle} & \blue{$\mathbb{R}^+$} & \blue{1 meter} &  \blue{$d_{thresh} = 2. r_{rob} + 0.5 $ meters. We observed that $d_{thresh} \sim 1 $ meters for our Turtlebot robot.  We defined this robot-dependant distance threshold after analyzing the point cloud intensity distribution for different transparent objects and by considering the safety clearance for a robot during navigation. Significantly smaller values can increase the risk of collisions with transparent obstacles.}\\
\hline
\blue{$m$}  & \blue{Side length of an ROI} & \blue{$\mathbb{Z}^+$} & \blue{100} & \blue{ROI is chosen such that it only covers the robot's nearby vicinity ($\sim 2.5$ meters). Higher values closer to $n$ could lead to artifacts (even though we are de-noising the ROI) in the final cost map (after the summation of multiple cost maps) since the ROI will include noisy objects further away from the robot.}    \\
\hline
\blue{$r_{rob}$}  & \blue{Radius of the robot} & \blue{$\mathbb{Z}^+$} & \blue{0.25} & \blue{We use a Turtlebot 2 robot with a 0.25 meter radius for our experiments.}  \\
\hline
\end{tabular}

\caption{\rev{Details of the parameters used in our TOPGN approach.}}
\label{Tab:Results2}
\end{table*}

\end{document}